\documentclass{article}
\usepackage[height=8in,width=6.5in]{geometry}
\usepackage{times}
\usepackage{amsmath,amsthm}
\usepackage[sort,numbers]{natbib}

\ifx\pdfoutput\undefined
\usepackage{graphicx}
\usepackage{epsfig}
\else
\usepackage[pdftex]{graphicx}
\usepackage{epstopdf}
\fi 

\usepackage[boxed,vlined]{algorithm2e}
\usepackage{xspace}
\usepackage{paralist}
\usepackage{array}

\input{header}

\newcommand{\iS}{\Sigma^{-1}}
\newtheorem{theorem}{Theorem}

\newtheorem{prop}[theorem]{Proposition}

\theoremstyle{definition}
\newtheorem{defn}[theorem]{Definition}

\newcommand{\sice}{{\small\textsf{SICE}}\xspace}
\newcommand{\pg}{{\small\textsf{PG}}\xspace}
\newcommand{\aspg}{{\small\textsf{ASPG}}\xspace}
\newcommand{\anes}{{\small\textsf{ANES}}\xspace}
\newcommand{\scov}{{\small\textsf{SCOV}}\xspace}
\newcommand{\matlab}{\textsc{Matlab}}

\newcommand{\normll}[1]{\norm{#1}{1,1}}

\newcommand{\proj}[2]{\mathrm{P}_{#1}\left(#2\right)}
\newcommand{\projU}[1]{\proj{\mathcal{U}}{#1}}
\newcommand{\aset}{\mathcal{A}}  
\newcommand{\bset}{\mathcal{B}}  
\newcommand{\kset}{\mathcal{K}}  
\newcommand{\sym}{\mathcal{S}}   
\newcommand{\primal}{\mathcal{F}}  
\newcommand{\dual}{\mathcal{G}}  
\newcommand{\abs}[1]{\left|#1\right|} 

\newcommand{\email}[1]{\texttt{#1}}

\title{Sparse Inverse Covariance Estimation via an Adaptive Gradient-Based Method}
\author{\begin{tabular}{ccc}
    Suvrit Sra\thanks{This work was done when SS was affiliated with the Max Planck Institute for Biological Cybernetics. We finished this preprint on June 3, 2010, but have uploaded it after more than a year to arXiv (in June 2011), so it is not the definitive version of this work.} &\hskip 12pt &Dongmin Kim\\
    Max Planck Institute for Intelligent Systems & & University of Texas at Austin\\
    T\"ubingen, Germany & & Austin, TX 78712\\
    \email{suvrit@tuebingen.mpg.de} & & \email{dmkim@cs.utexas.edu}
  \end{tabular}}
\date{version of: June 3, 2010}
\begin{document}
\maketitle

\begin{abstract}
  We study the problem of estimating from data, a sparse approximation to the
  inverse covariance matrix. Estimating a sparsity constrained inverse
  covariance matrix is a key component in Gaussian graphical model learning,
  but one that is numerically very challenging. We address this challenge by
  developing a new adaptive gradient-based method that carefully combines
  gradient information with an adaptive step-scaling strategy, which results
  in a scalable, highly competitive method. Our algorithm, like its
  predecessors, maximizes an $\ell_1$-norm penalized log-likelihood and has
  the same per iteration arithmetic complexity as the best methods in its
  class. Our experiments reveal that our approach outperforms state-of-the-art
  competitors, often significantly so, for large problems.
\end{abstract}

\section{Introduction}
A widely employed multivariate analytic tool is a Gaussian Markov Random Field
(GMRF), which in essence, simply defines Gaussian distributions over an
undirected graph. GMRFs are enormously useful; familiar applications include
speech recognition~\citep{bilmes}, time-series analysis, semiparametric
regression, image analysis, spatial statistics, and graphical model
learning~\citep{lauritzen,rue,waJo08}.

In the context of graphical model learning, a key objective is to learn the
structure of a GMRF~\citep{rue,waJo08}. Learning this structure often boils
down to discovering \emph{conditional} independencies between the variables,
which in turn is equivalent to locating zeros in the associated inverse
covariance (precision) matrix $\iS$---because in a GMRF, an edge is present
between nodes $i$ and $j$, if and only if $\iS_{ij} \neq
0$~\citep{rue,lauritzen}. Thus, estimating an inverse covariance matrix with
zeros is a natural task, though formulating it involves two obvious, but
conflicting goals: (i) \emph{parsimony} and (ii) \emph{fidelity}.

It is easy to acknowledge that fidelity is a natural goal; in the GMRF setting
it essentially means obtaining a statistically consistent estimate (e.g.~under
mild regularity conditions a maximum-likelihood estimate~\citep{mood}). But
parsimony too is a valuable target. Indeed, it has long been recognized that
sparse models are more robust and can even lead to better generalization
performance~\citep{dempster}. Furthermore, sparsity has great practical value,
because it: (i) reduces storage and subsequent computational costs
(potentially critical to settings such as speech recognition~\citep{bilmes});
(ii) leads to GMRFs with fewer edges, which can speed up inference
algorithms~\citep{malitov,waJo08}; (iii) reveals interplay between variables
and thereby aids structure discovery, e.g., in gene analysis~\citep{dobra}.

Achieving parsimonious estimates while simultaneously respecting fidelity is
the central aim of the Sparse Inverse Covariance Estimation (SICE)
problem. Recently,~\citet{aspre08} formulated SICE by considering the
penalized Gaussian log-likelihood~\citep{aspre08} (omitting constants)
\begin{equation*}
  \log\det(X) - \trace(SX) - \rho\nlsum_{ij} |x_{ij}|,
\end{equation*}
where $S$ is the sample covariance matrix and $\rho > 0$ is a parameter that
trades-off between fidelity and sparsity. Maximizing this penalized
log-likelihood is a non-smooth concave optimization task that can be
numerically challenging for high-dimensional $S$ (largely due to the
positive-definiteness constraint on $X$).  We address this challenge by
developing a new adaptive gradient-based method that carefully combines
gradient information with an adaptive step-scaling strategy. And as we will
later see, this care turns out to have a tremendous positive impact on
empirical performance. However, before describing the technical details, we
put our work in broader perspective by first enlisting the main contributions
of this paper, and then summarizing related work.

\vspace*{-5pt}
\subsection{Contributions}
\vspace*{-5pt}%
This paper makes the following main contributions:
\begin{enumerate}
\vspace*{-5pt}%
  \setlength{\itemsep}{-1pt}
\item It develops and efficient new gradient-based algorithm for SICE.
\item It presents associated theoretical analysis, formalizing key properties
  such as convergence.
\item It illustrates empirically not only our own algorithm, but also
  competing approaches.
\end{enumerate}
Our new algorithm has $O(n^3)$ arithmetic complexity per iteration like other
competing approaches. However, our algorithm has an edge over many other
approaches, because of three reasons: (i) the constant within the big-Oh is
smaller, as we perform only one gradient evaluation per iteration; (ii) the
eigenvalue computation burden is reduced, as we need to only occasionally
compute positive-definiteness; and (ii) the empirical convergence turns out to
be rapid. We substantiate the latter observation through experiments, and note
that state-of-the-art methods, including the recent sophisticated methods
of~\citet{lu09,lu10}, or the simple but effective method of~\citep{duchi}, are
both soundly outperformed by our method.

\noindent\paragraph{Notation:} We denote by $\sym^n$ and $\sym_+^n$, the set
of $n \times n$ symmetric, and semidefinite ($\succeq 0$) matrices,
respectively. The matrix $S \in \sym_+^n$ denotes the sample covariance matrix
on $n$ dimensional data points. We use $|X|$ to denote the matrix with entries
$|x_{ij}|$. Other symbols will be defined when used.

\subsection{Related work, algorithmic approaches}
\label{sec:backgr}
\citet{dempster} was the first to formally investigate SICE, and he proposed
an approach wherein certain entries of the inverse covariance matrix are
explicitly set to zero, while the others are estimated from data. Another old
approach is based on greedy forward-backward search, which is impractical as
it requires $O(n^2)$ MLEs~\citep{rue,lauritzen}. More recently,~\citet{mein06}
presented a method that solves $n$ different Lasso subproblems by regressing
variables against each other: this approach is fast, but approximate, and it
does not yield the maximum-likelihood estimator~\citep{glasso,aspre08}.

An $\ell_1$-penalized log-likelihood maximization formulation of SICE was
considered by~\citep{aspre08} who introduced a nice block-coordinate descent
scheme (with $O(n^4)$ per-iteration cost); their scheme formed the basis of
the \emph{graphical lasso} algorithm~\citep{glasso}, which is much faster for
very sparse graphs, but unfortunately has unknown theoretical complexity and
its convergence is unclear~\citep{lu09,lu10}. There exist numerous other
approaches to SICE, e.g., the Cholesky-decomposition (of the inverse
covariance) based algorithm of~\citep{rothman}, which is an expensive $O(n^4)$
per-iteration method. There are other Cholesky-based variants, but these are
either heuristic~\citep{huang}, or limit attention to matrices with a known
banded-structure~\citep{roth08}. Many generic approaches are also possible,
e.g., interior-point methods~\citep{ipbook} (usually too expensive for SICE),
gradient-projection~\citep{gpm,bertsekas,duchi,spg,lu10}, or projected
quasi-Newton~\citep{lbfgs,markpqn}. 

Due to lack of space we cannot afford to give the various methods a survey
treatment, and restrict our attention to the recent algorithms
of~\citep{lu09,lu10} and of~\citep{duchi}, as these seem to be over a broad
range of inputs the most competitive methods. \citet{lu09} optimizes the
dual~\eqref{eq:15} by applying Nesterov's techniques~\citep{nest05}, obtaining
an $O(1/\sqrt{\epsilon})$ iteration complexity algorithm (for
$\epsilon$-accuracy solutions), that has $O(n^3)$ cost per iteration. He
applies his framework to solve~\eqref{eq:sice} (with $R=\rho ee^T$), though to
improve the empirical performance he derives a variant with slightly weaker
theoretical guarantees. In~\citep{lu10}, the author derives two ``adaptive''
methods that proceed by solving a sequence of SICE problems with varying
$R$. His first method is based on an adaptive spectral projected
gradient~\citep{spg}, while second is an adaptive version of his Nesterov
style method from~\citep{lu09}. \citet{duchi} presented a simple, but
surprisingly effective gradient-projection algorithm that uses backtracking
line-search initialized with a ``good'' guess for the step-size (obtained
using second-order information). We remark, however, that even though
\citeauthor{duchi}'s algorithm~\citep{duchi} usually works well in practice,
its convergence analysis has a subtle problem: the method ensures descent, but
\emph{not} sufficient descent, whereby the algorithm can ``converge'' to a
non-optimal point. Our express goal is the paper is to derive an algorithm
that not only works better empirically, but is also guaranteed to converge to
the true optimum.

\subsection{Background: Problem formulation}
We begin by formally defining the SICE problem, essentially using a
formulation derived from~\citep{aspre08,duchi}:
\begin{equation}
  \label{eq:sice}
  \min_{X \succ 0}\quad \primal(X) = \trace(SX)-\log\det(X) + \trace(R|X|).
\end{equation}
Here $R \in \sym^n$ is an elementwise non-negative matrix of \emph{penalty}
parameters; if its $ij$-entry $r_{ij}$ is large, it will shrink $|x_{ij}|$
towards zero, thereby gaining sparsity. Following~\citep{lu10} we impose the
restriction that $S + \Diag(R) \succ 0$, where $\Diag(R)$ is the diagonal
(matrix) of $R$. This restriction ensures that~\eqref{eq:sice} has a solution,
and subsumes related assumptions made in~\citep{aspre08,duchi}. Moreover, by
the strict convexity of the objective, this solution is unique.

Instead of directly solving~\eqref{eq:sice}, we also
(like~\citep{aspre08,duchi,lu09,lu10}) focus on the dual as it is
differentiable, and has merely bound constraints. A simple device to derive dual problems to $\ell_1$-constrained problems
is~\citep{l1ls}: introduce a new variable $Z=X$ and a corresponding dual
variable $U$. Then, the Lagrangian is
\begin{equation}
  \label{eq:5}
  L(X,Z,U) = \trace(SZ)-\log\det(Z) + \normll{R \odot X} + \trace(U(Z-X)).
\end{equation}
The dual function $g(U) = \inf_{X,Z} L(X,Z,U)$. Computing this infimum is easy
as it separates into infima over $X$ and $Z$. We have
\begin{equation}
  \label{eq:6}
  \inf_{Z}\quad\trace(SZ)-\log\det(Z) + \trace(UZ),
\end{equation}
which yields
\begin{equation}
\label{eq.1}
  S - Z^{-1} + U = 0,\quad\text{or}\quad Z = (S + U)^{-1}. 
\end{equation}
Next consider
\begin{equation}
  \label{eq:7}
  \inf_X\quad\normll{R \odot X} - \trace(UX) =
  \inf_X\quad\nlsum_{ij}\rho_{ij}|x_{ij}| - \nlsum_{ij} u_{ij}x_{ij},
\end{equation}
which is unbounded unless $|u_{ij}| \le \rho_{ij}$, in which case it equals
zero. This, with~\eqref{eq.1} yields the dual function
\begin{equation}
  \label{eq:8}
  g(U) =
  \begin{cases}
    \trace((S+U)Z) - \log\det(Z) & |u_{ij}| \le \rho_{ij},\\
    -\infty & \text{otherwise}.
  \end{cases}
\end{equation}
Hence, the dual optimization problem may be written as (dropping constants)
\begin{equation}
  \label{eq:15}
  \begin{split}
    \min_U\quad &-\log\det(S+U)\\
    \text{s.t.}\quad& |u_{ij}| \le \rho_{ij},\quad i,j \in [n].
  \end{split}
\end{equation}

Observe that $-\nabla\dual = (S+U)^{-1} = X$; so we easily obtain the primal
optimal solution from the dual-optimal. Defining, $X_U = (S+U)^{-1}$, the
duality gap may be computed as
\begin{equation}
  \label{eq:1}
  \primal(X_U) - \dual(U) = \trace(SX_U) + \trace(R|X_U|) - n.
\end{equation}
Given this problem formalization we are now ready to describe our algorithm.

\section{Algorithm}
\label{sec:algorithm}
Broadly viewed, our new adaptive gradient-based algorithm consists of three
main components: (i) gradient computation; (ii) step-size computation; and
(iii) ``checkpointing'' for positive-definiteness. Computation of the gradient
is the main $O(n^3)$ cost that we (like other gradient-based methods), must
pay at every iteration. Even though we compute gradients only once per
iteration, saving on gradient computation is not enough. We also need a better
strategy for selecting the step-size and for handling the all challenging
positive-definiteness constraint. Unfortunately, both are easier said than
done: a poor choice of step-size can lead to painfully slow
convergence~\citep{bertsekas}, while enforcing positive-definiteness requires
potentially expensive eigenvalue computations. The key aspect of our method is
a new step-size selection strategy, which helps greatly accelerate empirical
convergence. We reduce the unavoidable cost of eigenvalue computation by not
performing it at every iteration, but only at certain ``checkpoint''
iterations. Intuitively, at a checkpoint the iterate must be tested for
positive-definiteness, and failing satisfaction an alternate strategy (to be
described) must be invoked. We observed that in practice, most intermediate
iterates generated by our method usually satisfy positive-definiteness, so the
alternate strategy is triggered only rarely. Thus, the simple idea of
checkpointing has a tremendous impact on performance, as borne through by our
experimental results.  Let us now formalize these ideas below.

We begin by describing our stepsize computation, which is derived from the
well-known Barzilai-Borwein (BB) stepsize~\citep{bb} that was found to
accelerate (\emph{unconstrained}) steepest descent
considerably~\citep{bb,Raydan:BB:1993,Dai:PBB:2005}. Unfortunately the BB
stepsizes do \emph{not} work out-of-the-box for \emph{constrained} problems,
i.e.,~nai\"ively plugging them into gradient-projection does not
work~\citep{Dai:PBB:2005}. For constrained setups the spectral projected
gradient (SPG) algorithm~\citep{spg} is a popular method based on BB stepsizes
combined with a globalization strategy such as non-monotone
line-searches~\citep{Grippo:BB:2002}. For SICE, very recently~\citet{lu10}
exploited SPG to develop a sophisticated, adaptive SPG algorithm for SICE.
For deriving our BB-style stepsizes, we first recall that for unconstrained
\emph{convex quadratic} problems, steepest-descent with BB is guaranteed to
converge~\citep{Friedlander:BB:1995}. This immediately suggests using an
\emph{active-set} idea, which says that if we could identify variables active
at the solution, we could reduce the optimization to an unconstrained problem,
which might be ``easier'' to solve than its constrained counterpart.  We thus
arrive at our first main idea: \emph{use the active-set information to modify
  the computation of the BB step itself}. This simple idea turns out to have a
big impact on empirical consequences; we describe it below.

We partition the variables into \emph{active} and \emph{working} sets, and
carry out the optimization over the working set alone. Moreover, since
gradient information is also available, we exploit it to further refine the
active-set to obtain the \emph{binding} set; both are formalized below.
\begin{defn}
  Given dual variable $U$, the \emph{active} set $\aset(U)$ and the
  \emph{binding} set $\bset(U)$ \mbox{are~defined~as}
\begin{align}
  \aset(U) &= \big\{ (i,j) \ \big|\ \abs{U_{ij}} = \rho_{ij}
  \big\}, \label{eq:aset}, \\ 
  \bset(U) &= \big\{ (i,j) \ \big|\ U_{ij}= -\rho_{ij},\  \partial_{ij}
  \dual(U) > 0,\quad\text{or}\quad
  U_{ij}= \rho_{ij},\  \partial_{ij} \dual(U) < 0 \big\}, \label{eq:bset}  
\end{align}
\end{defn}
where $\partial_{ij}\dual(U)$ denotes the $(i,j)$ component of
$\nabla\dual(U)$. The role of the binding set is simple: variables bound at
the current iteration are guaranteed to be active at the next. Specifically,
let $\Uc =  \set{U \ \big|\ \abs{U_{ij}} = \rho_{ij}}$, and denote orthogonal
projection onto $\mathcal{U}$ by $\projU{\cdot}$, i.e., 
\begin{equation}
\projU{U_{ij}} = \mathrm{mid}\{ U_{ij},\ -\rho_{ij},\ \rho_{ij}\}.
\end{equation}
If $(i,j)\in\bset(U^k)$ and we iterate $U^{k+1} = \projU{U^k - \gamma^k \nabla
  \dual(U^k)}$ for some $\gamma^k > 0$, then since
\[ \abs{U^{k+1}_{ij}} = \abs{ \projU{U^k_{ij} - \gamma^k \partial_{ij}
    \dual(U^k)} } = \rho_{ij}, \] the membership $(i,j) \in \aset(U^{k+1})$
holds. Therefore, if we know that $(i,j)\in\bset(U^k)$, we may discard
$U_{ij}^k$ from the update. We now to employ this idea in conjunction with the
BB step, and introduce our new modification: first compute $\bset(U^k)$ and
then confine the computation of the BB step to the \emph{subspace} defined by
$(i,j)\notin\bset(U^k)$. The last ingredient that we need is \emph{safe-guard}
parameters $\sigma_{\min}, \sigma_{\max}$, which ensure that the BB-step
computation is well-behaved~\citep{lu10,Grippo:BB:2002,ray97}. We thus obtain
our modified computation, which we call \emph{modified-BB} (MBB) step:
\begin{defn}[Modified BB step]
\begin{equation}
\begin{split}
  \label{eq:rbb}
  \tilde{\sigma} = \frac{\norm{\Delta\tilde{U}^{k-1}}{F}^2}
    {\sum_i\sum_j{\Delta\tilde{U}^{k-1}_{ij}\cdot
        \Delta\tilde{\dual}^{k-1}_{ij}}} \cdots&\cdots
  (a), \qquad\text{or}\qquad 
\tilde{\sigma} =
\frac{\sum_i\sum_j{\Delta\tilde{U}^{k-1}_{ij} \cdot
    \Delta\tilde{\dual}^{k-1}_{ij}}}
{\norm{\Delta\tilde{\dual}^{k-1}}{F}^2}
\cdots\cdots(b),  \\
\sigma^k &= \mathrm{mid}\{\sigma_{min},\ \tilde{\sigma},\ \sigma_{max}\}.
\end{split}
\end{equation}
where $\Delta\tilde{U}^{k-1}$ and $\Delta\tilde{\dual}^{k-1}$ are defined by
over only non-bound variables by the following formulas:
\begin{align}
\Delta\tilde{U}^{k-1}_{ij} &= \left\{ \begin{array}{cl}
[U^{k-1} - U^{k-2}]_{ij}, & \text{if}\quad (i,j) \notin \bset(U^k), \\
0, & \text{otherwise}. 
\end{array}\right. \\
\Delta\tilde{\dual}^{k-1}_{ij} &= \left\{ \begin{array}{cl}
[\nabla\dual^{k-1} - \nabla\dual^{k-2}]_{ij}, & \text{if}\quad
(i,j) \notin \bset(U^k), \\ 
0, & \text{otherwise}. 
\end{array}\right.
\end{align}
\end{defn}

Assume for the moment that we did not have the positive-definite
constraint. Even then, our MBB steps derived above, are not enough to ensure
convergence. This is not surprising since even for the \emph{unconstrained}
general convex objective case, ensuring convergence of BB-step-based methods,
globalization (line-search) strategies are
needed~\citep{Grippo:BB:2002,ray97}; the problem is only worse for constrained
problems too, and some line-search is needed to ensure
convergence~\citep{Friedlander:BB:1995,spg,Grippo:BB:2002}. The easy solution
for us too would be invoke a non-monotone line-search strategy. However, we
observed that MBB step alone often produces converging iterates, and even lazy
line-search in the style of~\citep{Dai:PBB:2005,Grippo:BB:2002} affects it
adversely. Here we introduce our second main idea: To retain empirical
benefits of subspace steps while still guaranteeing convergence, we propose to
scale the MBB stepsize using a diminishing scalar sequence, but \emph{not}
out-of-the-box; instead we propose to relax the application of diminishing
scalars by introducing an ``optimistic'' diminishment strategy. Specifically,
we scale the MBB step~\eqref{eq:rbb} with some constant scalar $\delta$ for a
fixed number, say $M$, of iterations. Then, we check whether a \emph{descent
  condition} is satisfied. If the current iterate passes the descent test,
then the method continues for another $M$ iterations with the \emph{same}
$\delta$; if it fails to satisfy the descent, then we diminish the scaling
factor $\delta$. This diminishment is ``optimistic'' because even when the
method fails to satisfy the descent condition that triggered diminishment, we
merely diminish $\delta$ once, and continue using it for the next $M$
iterations. We remark that superficially this strategy might seem similar to
occasional line-search, but it is fundamentally different: unlike line-search
our strategy does \emph{not} enforce monotonicity after failing to descend for
a prescribed number of iterations. We formalize this below.
 
Suppose that the method is at iteration $c$, and from there, iterates with a
constant $\delta^c$ for the next $M$ iterations, so that from the current
iterate $U^c$, we compute
\begin{equation}
\label{eq:asbb}
  U^{k+1} = \projU{U^{k} - \delta^c\cdot\sigma^{k} \nabla \dual(U^{k})},
\end{equation}
for $k=c,c+1,\cdots, c+M-1$, where $\sigma^{k}$ is computed
via~\eqref{eq:rbb}-(a) and~\eqref{eq:rbb}-(b) alternatingly. 
Now, at the iterate $U^{c+M}$, we check the \emph{descent
  condition:}
\begin{defn}[Descent condition]
\begin{equation}
\label{eq:descent}
  \dual(U^c)-\dual(U^{c+M}) \ge \kappa \sum_{i}\sum_{j} \partial_{ij}
  \dual(U^c)\cdot[U^c - U^{c+M}]_{ij}, 
\end{equation}
for some fixed parameter $\kappa \in (0, 1)$. 
\end{defn}
If iterate $U^{c+M}$ passes the test~\eqref{eq:descent}, then we reuse
$\delta^c$ and set $\delta^{c+M} = \delta^c$; otherwise we diminish $\delta^c$
and set $\delta^{c+M} \gets \eta\cdot\delta^c$, for some fixed $\eta \in
(0,1)$. After adjusting $\delta^{c+M}$ the method repeats the
update~\eqref{eq:asbb} for another $M$ iterations. Finally, we measure the
duality gap every $M$ iterations to determine termination of the
method. Explicitly, given a stopping threshold $\epsilon > 0$, we define:
\begin{defn}[Stopping criterion]
\begin{equation}
  \label{eq:term}
  \primal(-\nabla\dual(U)) - \dual(U) = \primal(X_U) - \dual(U) =
  \trace(SX_U) + \trace(R\abs{X_U}) - n < \epsilon. 
\end{equation}
\end{defn}
Algorithm~\ref{alg:sice} encapsulates all the above details, and presents
pseudo-code of our algorithm \sice.
\begin{center}
\begin{minipage}{0.8\linewidth}
  \setlength{\algomargin}{5pt}
  \begin{algorithm}[H]
    \caption{Sparse Inverse Covariance Estimation (\sice).}
    \label{alg:sice}
    Given $U^0$ and $U^1$\;
    \For{$i=1,\cdots$ until the stopping criterion~\eqref{eq:term} met}{
      $\tilde{U}^0 \gets U^{i-1}$ and $\tilde{U}^1 \gets U^i$\;
      \For(/* Modified BB */){$j=1,\cdots, M$}{
        Compute $\bset(\tilde{U}^j)$ then compute $\sigma^j$
        using~\eqref{eq:rbb}-(a) and -(b)
        alternatingly\;
        $\tilde{U}^{j+1} \gets \projU{\tilde{U}^j -
          \delta^i\cdot\sigma^j \nabla\dual(\tilde{U}^j) }$\;
      }
      \eIf{Is $S+\tilde{U} \succ 0$}{\eIf {$\tilde{U}^M$ satisfies~\eqref{eq:descent}}{ 
        $U^{i+1} \gets \tilde{U}^M,$ and $\delta^{i+1} \gets \delta^i$ \;
      }(        /* Diminish Optimistically */){
        $\delta^{i+1} \gets \eta\delta^i$, where $\eta \in (0, 1)$\;
      }}{Enforce posdef\;
      $\tilde{U}^M \gets $\;
      \eIf{$\tilde{U}^M$ satisfies~\eqref{eq:descent}}{$U^{i+1} \gets \tilde{U}^M,$ and $\delta^{i+1} \gets \delta^i$\;}(        /* Diminish Optimistically */){$\delta^{i+1} \gets \eta\delta^i$, where $\eta \in (0, 1)$\;}
    }}
  \end{algorithm}
\end{minipage}
\end{center}

\section{Analysis}
\label{sec:analysis}
In this section we analyze theoretical properties of \sice. First, we
establish convergence, and then briefly discuss some other properties.

For clarity, we introduce some additional notation. Let $M$ be the fixed
number of modified-BB iterations, i.e., descent condition~\eqref{eq:descent}
is checked only every $M$ iterations. We index these $M$-th iterates with
$\kset = \{1,\ M+1,\ 2M+1,\ 3M+1,\ \cdots\}$, and then consider the sequence
$\{U^r\},\ r\in\kset$ generated by \sice.  Let $U^*$ denote the optimal
solution to the problem. We first show that such an optimal exists.
\begin{theorem}[Optimal]
  If $R > 0$, then there exists an optimal $U^*$ to~\eqref{eq:15} such
  that $$X^*(S +U^*) = \mi,\quad\text{and}\quad\trace(X^*U^*) =
  \trace(R|X^*|).$$
\end{theorem}
\begin{proof}
  From \citep[Proposition~2.2]{lu10} (or the supplementary material), We know
  that Problem~\eqref{eq:sice} has a unique optimal solution $X^* \in
  \posdef{n}$. On the other hand, from \citep[Lemma~1]{duchi}, there exists a
  feasible point $U \in \text{int}(\Uc)$, thereby Slater's constraint
  qualification guarantees that the theorem holds.
\end{proof}

We remind the reader that when proving convergence of iterative optimization
routines, one often assumes Lipschitz continuity of the objective. We,
therefore begin our discussion by stating the fact that the dual objective
function $\dual(U)$ is indeed Lipschitz continuous.
\begin{prop}[Lipschitz constant~\citep{lu09}]
\label{prop:lip}
The dual gradient $\nabla\dual(U)$ is Lipschitz continuous on $\Uc$ with
Lipschitz constant $L=\lambda_{\max}^2(X^*)(\max_{ij} r_{ij})^2$.
\end{prop}
We now prove that $\{U^r\} \rightarrow U^*$ as $r\rightarrow\infty$. To that
end, suppose that the diminishment step $\delta^{c+M} \gets \eta\cdot\delta^c$
is triggered only a finite number of times. Then, there exists a sufficiently
large $K$ such that for all $r\in\kset,\ r\ge K$,
\begin{equation*} 
  \dual(U^r)-\dual(U^{r+1}) \ge \kappa \sum_{i}\sum_{j} \partial_{ij}
  \dual(U^r)\cdot[U^r - U^{r+1}]_{ij}.
\end{equation*}
In such as case, we may view the sequence $\{U^r\}$ as if it were generated by
an ordinary gradient projection scheme, whereby the convergence $\{U^r\} \to
U^*$ follows using standard arguments~\citep{bertsekas}. Therefore, to prove
the convergence of the entire algorithm, it suffices to discuss the case where
the diminishment occurs infinitely often. Corresponding to an infinite
sequence $\{U^r\}$, there is a sequence $\{\delta^r\}$, which by construction
is diminishing. Hence, if we impose the following \emph{unbounded sum
  condition} on this sequence, following~\citep{bertsekas}, we can again
ensure convergence.
\begin{defn}[Diminishing with unbounded sum]
\label{defn:dim}
  \begin{equation*}
    \text{(i)}\ \lim_{r\to\infty}\delta^r = 0,\quad\text{and}\quad
    \text{(ii)}\ \lim_{r\to\infty}\nlsum_{i=1}^r\delta^i = \infty.
  \end{equation*}
\end{defn}
Since in our algorithm, we scale the MBB step~\eqref{eq:rbb} $\sigma^r$ by a
diminishing scalar $\delta^r$, we must ensure that the resulting sequence
$\delta^r\alpha^r$ also satisfies the above condition. This is easy, but for
the reader's convenience formalized below.
\begin{prop}
  \label{prop:step}
  In Algorithm~\ref{alg:sice}, the stepsize $\delta^r\cdot\sigma^r$ satisfies
  the condition~\ref{defn:dim}.
\end{prop}
\begin{proof}
  Since $\lim_{r\to\infty}\delta^r = 0$ and
  $\lim_{r\to\infty}\nlsum_{i=1}^r\delta^i = \infty$ by construction, 
  \begin{equation*}
    \lim_{r\to\infty}\delta^r\cdot\sigma^r\ \le\
    \sigma_{\max}\cdot \lim_{r\to\infty}\delta^r = 0
\quad\text{and}\quad
    \lim_{r\to\infty}\sum_{i=1}^r\delta^i\cdot\sigma^i\ \ge\
    \sigma_{\min}\cdot \lim_{r\to\infty}\sum_{i=1}^r\delta^i
    = \infty.\qedhere
  \end{equation*}
\end{proof}
Using this proposition we can finally state the main convergence theorem. Our
proof is based on that of gradient-descent; we adapt it for our problem by
showing some additional properties of $\{U^r\}$.

\begin{theorem}[Convergence]
  The sequence of iterates $(X^k, U^k)$ generated by Algorithm~1 converges to
  the optimum solution $(X^*,U^*)$.
\end{theorem}
\begin{proof}
  Consider the update~\eqref{eq:asbb} of Algorithm~\ref{alg:sice}, we can
  rewrite it as
  \begin{equation} 
    \label{eq:thm1}
    U^{r+1} = U^{r} + \delta^r\cdot\sigma^{r} D^r,
  \end{equation}
  where the \emph{descent direction} $D^r$ satisfies
  \begin{equation}
    \label{eq:thm2}
    D_{ij}^r = \begin{cases}
      0, & \text{if}\ (i,j) \in \bset(U^r), \\
       -\mathrm{mid}\left\{
         \frac{\displaystyle U_{ij}^r - \rho_{ij}}{\displaystyle \delta^r\cdot\sigma^r},\ 
         \frac{\displaystyle U_{ij}^r + \rho_{ij}}{\displaystyle \delta^r\cdot\sigma^r},\ 
         \partial_{ij} \dual(U^r)\right\}, &
       \text{otherwise}. 

    \end{cases}
\end{equation} 
Using~\eqref{eq:thm2} and the fact that there exists at least one
$\abs{\partial_{ij} \dual(U^r)} > 0$ unless $U^r=U^*$, we can conclude that
there exists a constant $c_1 > 0$ such that
\begin{equation}
\label{eq:thm3}
  -\sum_{i}\sum_{j}\partial_{ij} \dual(U^r)\cdot D_{ij}^r\ \ge\
  \sum_{(i,j)\notin\bset(U^r)} 
  m_{ij}^2 \ \ge\ c_1\norm{\nabla \dual(U^r)}{F}^2\ >\ 0,
\end{equation}
where $m_{ij} = \mathrm{mid}\Big\{
\frac{U_{ij}^r-\rho_{ij}}{ \delta^r\cdot\sigma^r},\ 
\frac{U_{ij}^r+\rho_{ij}}{ \delta^r\cdot\sigma^r},\ 
\partial_{ij} \dual(U^r)\Big\}$.\\
Similarly, we can also show that there
exists $c_2>0$ such that
\begin{equation}
  \label{eq:thm4}
  \norm{D^r}{F}^2 \le c_2 \norm{\nabla \dual(U^r)}{F}^2.
\end{equation}
Finally, using the inequalities~\eqref{eq:thm3} and \eqref{eq:thm4}, in
conjunction with Proposition~\ref{prop:lip},~\ref{prop:step}, the proof is
immediate from Proposition~1.2.4 in~\citep{bertsekas}.
\end{proof}

\section{Numerical Results}
\label{sec:expts}
\vspace*{-5pt}%
We present numerical results on both synthetic as well as real-world data. We
begin with the synthetic setup, which helps to position our method via-á-vis
the other competing approaches.  We compare against algorithms that represent
the state-of-the-art in solving the \sice problem. Specifically, we compare
the following 5 methods:
\begin{inparaenum}[(i)]
\setlength{\itemsep}{0pt}
\item \pg: the projected gradient algorithm of~\citep{duchi};\footnote{\scriptsize Impl.}
\item \aspg: the very recent adaptive spectral projected gradient algorithm~\citep{lu10};\footnotemark
\item \anes: an adaptive, Nesterov style algorithm from the \aspg
  paper~\citep{lu10};\footnotemark[2]\footnotetext{\scriptsize Downloaded
    from}
\item \scov: a sophisticated smooth optimization scheme built on Nesterov's
  techniques~\citep{lu09};\footnote{\scriptsize Downloaded from} and
\item \sice: our proposed algorithm.
\end{inparaenum}

We note at this point that although due to space limits our experimentation is
not exhaustive, it is fairly extensive. Indeed, we also tried several other
baseline methods such as: gradient projection with Armijo-search
rule~\citep{bertsekas}, limited memory projected quasi-Newton
approaches~\citep{markpqn,lbfgs}, among others. But all these approaches
performed across wide range of problems similar to or worse than \pg. Hence,
we do not report numerical results against them. We further note that
\citet{lu09} reports the \scov method vastly outperforms block-coordinate
descent (BCD) based approaches such as~\citep{aspre08}, or even the graphical
lasso FORTRAN code of~\citet{glasso}. Hence, we also omit BCD methods from our
comparisons. 

\vspace*{-5pt}
\subsection{Synthetic Data Experiments}
\label{sec:synthExpts}
\vspace*{-5pt}%
We generated random instances in the same manner as~\citet{aspre08}; the same
data generation scheme was also used by~\citet{lu09,lu10}, and to be fair to
them, we in fact used their \matlab code to generate our data (this code is
included in the supplementary material for the interested reader). First one
samples a sparse, invertible matrix $S' \in \sym^n$, with ones on the
diagonal, and having a desired density $\delta$ obtained by randomly adding
$+1$ and $-1$ off-diagonal entries. Then one perturbs $S'$ to obtain $S'' =
(S')^{-1}+ \tau N$, where $N \in \sym^n$ is an i.i.d.\ uniformly random
matrix. Finally, one obtains $S$, the random sample covariance matrix by
shifting the spectrum of $S''$ by
\begin{equation*}
  S = S'' - \min\set{\lambda_{\min}(S'')-\varsigma, 0}I,
\end{equation*}
where $\varsigma > 0$ is a small scalar. The specific parameter values used
were the same as in~\citep{lu09}, namely: $\delta=.01$, $\tau=.15$, and
$\varsigma=10^{-4}$. We generated matrices of size $100m \times 100m$, for
$m \in [1\ldots 10] \cup \set{16,20}$.

\begin{table}[htbp]\scriptsize
  \centering
  \begin{tabular}{cc}
    \begin{tabular}{r||r|r|r|r|>{\bf}r}
      $\epsilon$ & PG & ASPG & ANES & SCOV & SICE\\
      \hline
      1E-01 &8.46 & 38.97 & 48.05 & 21.75 & 2.28 \\
      1E-02 &34.89 & 72.25 & 100.61 & 46.10 & 4.74 \\
      1E-03 &127.37 & 142.47 & 214.56 & 98.73 & 7.61 \\
      1E-04 &807.16 &519.37 & 482.97 & 224.58 & 17.30 \\
      2E-05 &- &1914.33 & 862.86 & 409.06 & 42.38 \\
      1E-05 &- &2435.14 & 1149.04 & 548.34 & 52.29 \\
      \hline
    \end{tabular}&
    \begin{tabular}{r||r|r|r|r|>{\bf}r}
      $\epsilon$ & PG & ASPG & ANES & SCOV & SICE\\
      \hline
      1E-01 &128.61 & 1005.98 & 1818.00 & 1858.70 & 66.08 \\
      1E-02 &358.90 & 1712.01 & 3173.55 & 3216.55 & 102.48 \\
      1E-03 &833.86 & 3529.29 & 5411.49 & 5085.51 & 170.25 \\
      1E-04 &1405.56 & 5008.81 & 9528.91 & 9046.04 & 234.18 \\
      2E-05 &- &6420.18 & 15964.26 & 15428.23 & 307.67 \\
      1E-05 &- &7562.52 & 19812.06 & 19621.27 & 336.25 \\
      \hline
    \end{tabular}
  \end{tabular}
  \caption{\footnotesize Time (secs) for reaching desired duality
    gap. Left, $S \in \sym_+^{300}$, and right, $S \in \sym_+^{1000}$. Since all
    algorithms decrease the duality-gap non-monotonically, we  
    show the time-points when a prescribed $\epsilon$ is first attained.} 
  \label{tab:synth.gap}
\end{table}

Table~\ref{tab:synth.gap} shows two sample results (more extensive results are
available in the supplementary material). All the algorithms were run with a
convergence tolerance of $\epsilon=2\cdot 10^{-5}$; this tolerance is very
tight, as previously reported results usually used $\epsilon=.1$ or
$.01$. Table~\ref{tab:synth.gap} reveals two main results: (i) the advanced
algorithms of~\citep{lu09,lu10} are outperformed by both \pg and \sice; and
(ii) \sice vastly outperforms all other methods. The first result can be
attributed to the algorithmic simplicity of \pg and \sice, because \aspg,
\anes, and \scov employ a complex, more expensive algorithmic structure: they
repeatedly use eigendecompositions. The second result, that is, the huge
difference in performance between \pg and \sice (see
Table~\ref{tab:synth.gap}, second sub-table, fourth and fifth rows) is also
easily explained. Although the step-size heuristic used by \pg works very-well
for obtaining low accuracy solutions, for slightly stringent convergence
criteria its progress becomes much too slow. \sice on the other hand, exploits
the gradient information to rapidly identify the active-set, and thereby
exhibits faster convergence. In fact, \sice turns out to be even faster for
low accuracy ($\epsilon=.1$) solutions.

The convergence behavior of the various methods is illustrated in more detail
in Figure~\ref{fig:synth.gap}. The left panel plots the duality gap attained
as a function of time, which makes the performance differences between the
methods starker: \sice converges rapidly to a tight tolerance in a
\emph{fraction of the time} needed by other methods. The right panel
summarizes performance profiles as the input matrix size $n \times n$ is
varied. Even as a function of $n$, \sice is seen to be overall much faster
(y-axis is on a log scale) than other methods.
\begin{figure}[htbp]
  \centering
  \begin{tabular}{@{\hspace{-1pt}}l@{\hspace{5pt}}l||l@{\hspace{5pt}}l}
    \includegraphics[scale=0.23]{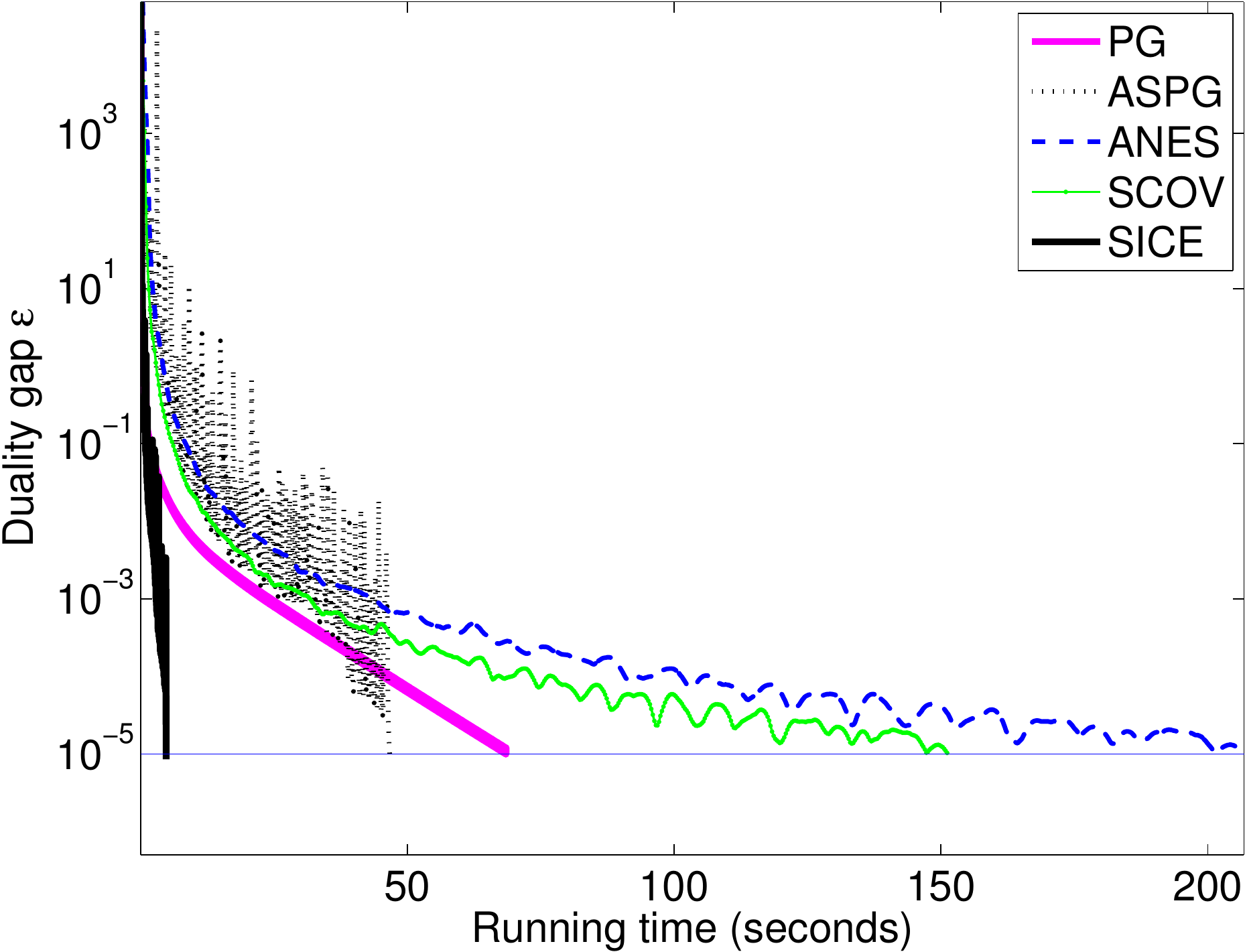} &
    \includegraphics[scale=0.23]{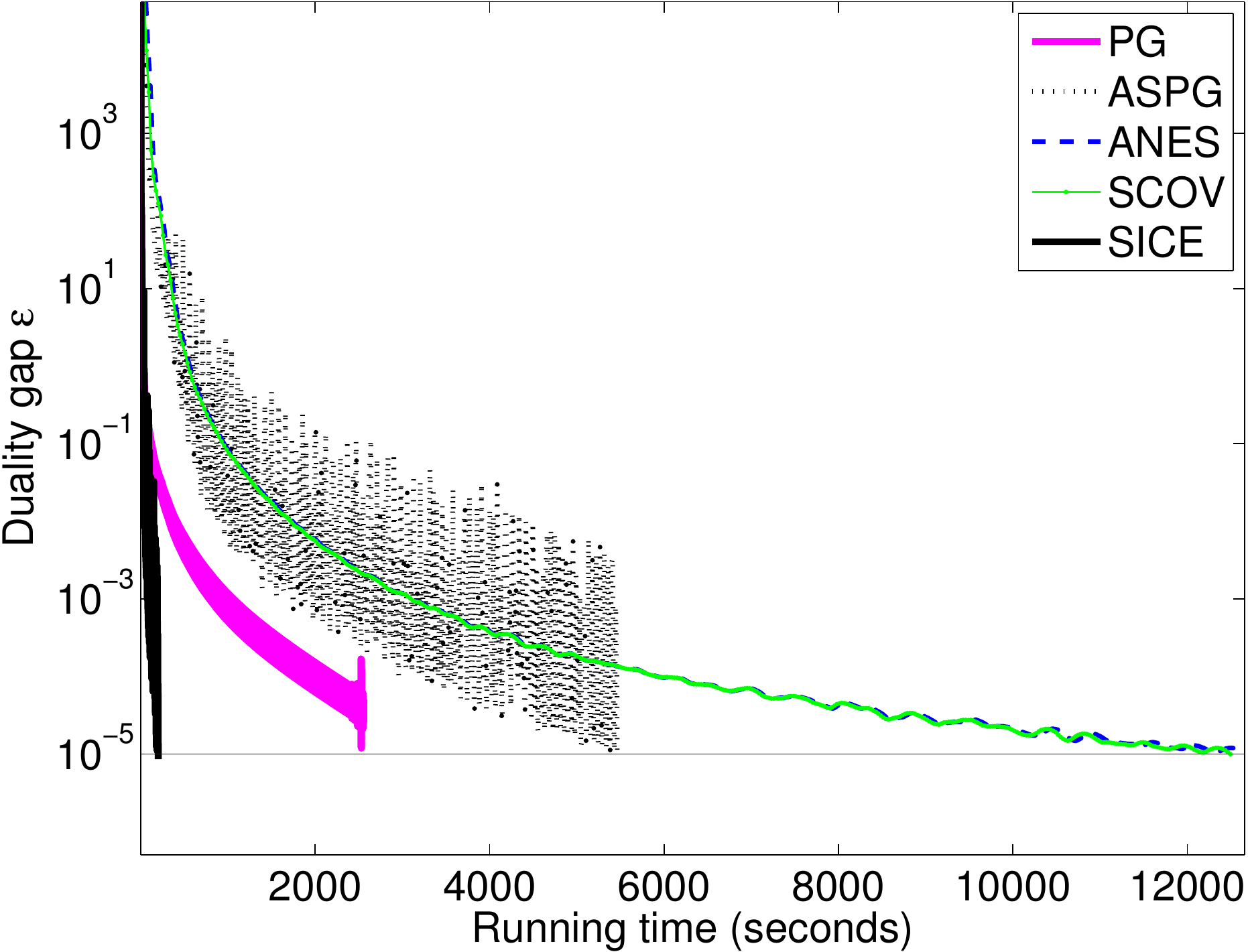} &
    \includegraphics[scale=0.2]{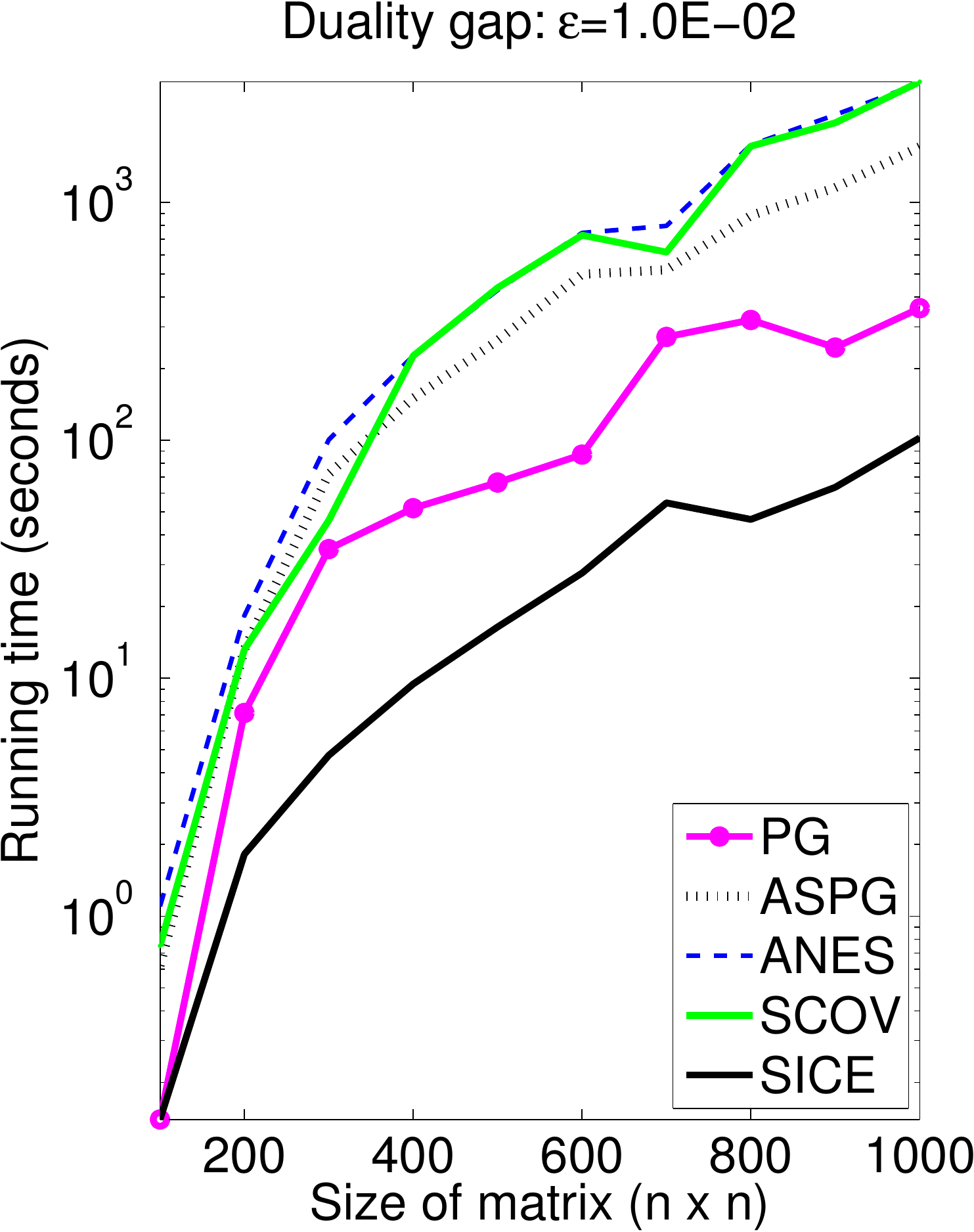} &
    \includegraphics[scale=0.2]{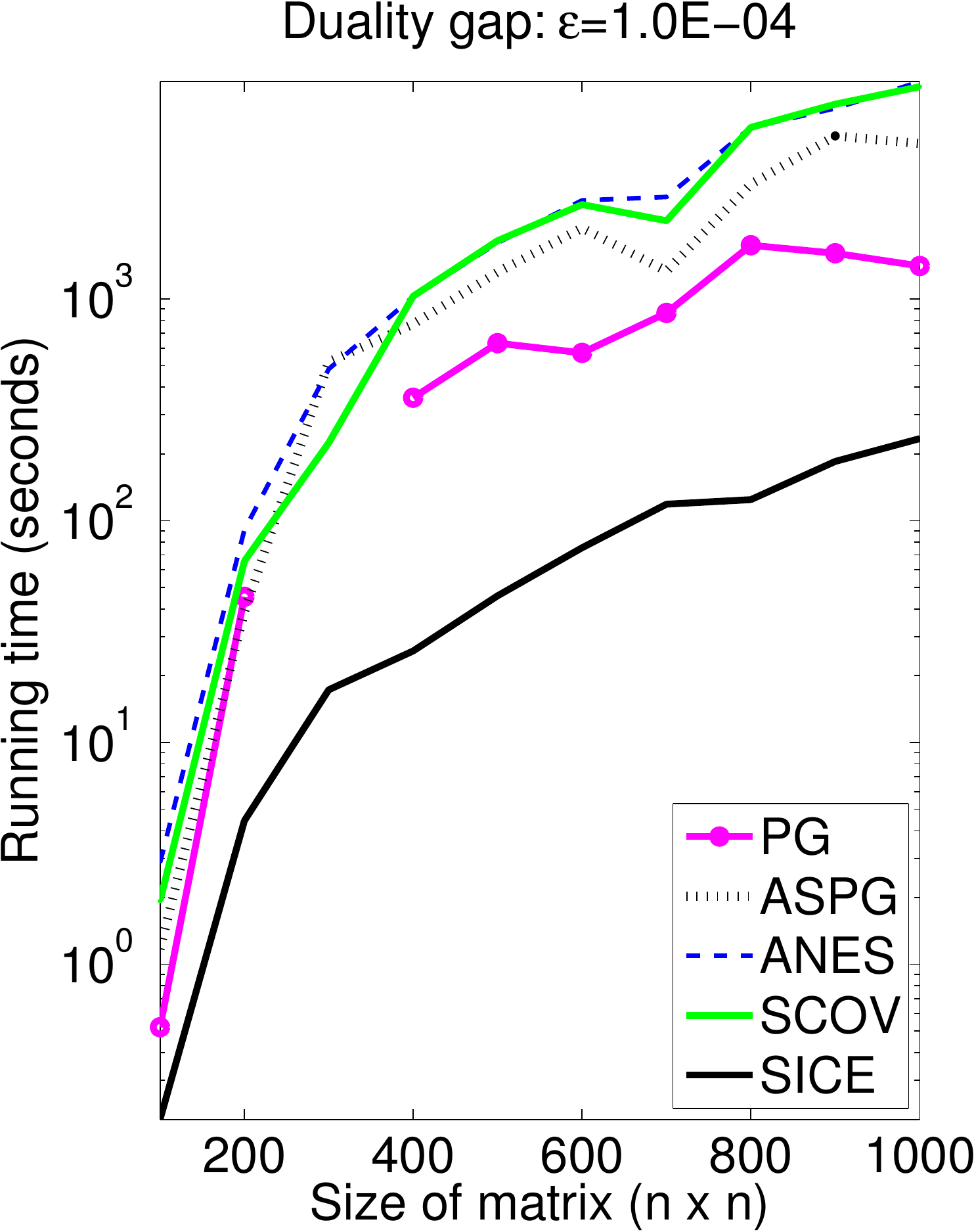}\\
    \hline
    \hline
  \end{tabular}
  \caption{\small Left panel: duality gap attained as a function of running
    time (in seconds); Right panel: Time taken to reach a prescribed duality
    gap as a function of matrix $S$'s size.}
\label{fig:synth.gap}
\end{figure}

\section{Conclusions and future work}
We developed, analyzed, and experimented with a powerful new algorithm for
estimating the sparse inverse covariance matrices. Our method was shown to
outperform all the competing state-of-the-art approaches. We attribute the
speedup gains to our careful algorithmic design, which also shows that even
though the spectral projected gradient method is fast and often successful,
invoking it out-of-the-box can be a suboptimal approach in scenarios where the
constraints were simple. In such as setup, our modified spectral approach
turns out to afford significant empirical gains. 

At this point several avenues of future work are open. We list them in order
of difficulty:
\begin{inparaenum}[(i)]
\item Implementing our method in a multi-core or GPU setting;
\item Extending our approach to handle more general constraints;
\item Deriving an even more efficient algorithm that only occasionally
  computes gradients;
\item Deriving a method for \sice that does not need to explicitly compute
  matrix inverses.
\end{inparaenum}

{\small
  \setlength{\bibsep}{0pt}

}

\appendix
\section{Additional numerical results}
\begin{minipage}{1.0\linewidth}
  \scriptsize\centering
  \begin{tabular}{cc}
    \begin{tabular}{r||r|r|r|r|>{\bf}r}
      $\epsilon$ & PG & ASPG & ANES & SCOV & SICE\\
      \hline
      1E-01 &0.09 & 0.51 & 0.74 & 0.51 & 0.12 \\
      1E-02 &0.14 & 0.62 & 1.10 & 0.75 & 0.14 \\
      1E-03 &0.26 & 0.94 & 1.73 & 1.16 & 0.16 \\
      1E-04 &0.52 & 1.17 & 2.87 & 1.93 & 0.20 \\
      2E-05 &0.70 & 1.23 & 3.77 & 2.52 & 0.20 \\
      1E-05 &0.78 & 1.26 & 4.31 & 2.86 & 0.20 \\
      \hline
    \end{tabular}
    &
    \begin{tabular}{r||r|r|r|r|>{\bf}r}
      $\epsilon$ & PG & ASPG & ANES & SCOV & SICE\\
      \hline
      1E-01 &1.70 & 6.40 & 8.48 & 6.04 & 0.94 \\
      1E-02 &7.15 & 13.23 & 18.33 & 13.21 & 1.82 \\
      1E-03 &22.79 & 27.00 & 43.98 & 31.79 & 2.97 \\
      1E-04 &45.38 & 38.91 & 90.84 & 65.76 & 4.44 \\
      2E-05 &61.46 & 46.53 & 162.22 & 118.36 & 4.80 \\
      1E-05 &68.50 & 46.69 & 206.81 & 151.32 & 4.81 \\
      \hline
    \end{tabular}
    \\
    $100 \times 100$ & $200 \times 200$
  \end{tabular}
\end{minipage}
\begin{minipage}{1.0\linewidth}
  \scriptsize\centering
  \begin{tabular}{cc}
    \begin{tabular}{r||r|r|r|r|>{\bf}r}
      $\epsilon$ & PG & ASPG & ANES & SCOV & SICE\\
      \hline
      1E-01 &8.46 & 38.97 & 48.05 & 21.75 & 2.28 \\
      1E-02 &34.89 & 72.25 & 100.61 & 46.10 & 4.74 \\
      1E-03 &127.37 & 142.47 & 214.56 & 98.73 & 7.61 \\
      1E-04 &807.16 &519.37 & 482.97 & 224.58 & 17.30 \\
      2E-05 &- &1914.33 & 862.86 & 409.06 & 42.38 \\
      1E-05 &- &2435.14 & 1149.04 & 548.34 & 52.29 \\
      \hline
    \end{tabular}
    &
    \begin{tabular}{r||r|r|r|r|>{\bf}r}
      $\epsilon$ & PG & ASPG & ANES & SCOV & SICE\\
      \hline
      1E-01 &16.27 & 84.71 & 109.95 & 109.04 & 5.42 \\
      1E-02 &51.89 & 150.26 & 226.52 & 226.34 & 9.46 \\
      1E-03 &139.73 & 338.29 & 471.29 & 473.33 & 15.59 \\
      1E-04 &357.93 & 769.44 & 1020.82 & 1027.13 & 25.84 \\
      2E-05 &529.69 & 997.44 & 1776.06 & 1793.75 & 28.43 \\
      1E-05 &603.84 & 1099.55 & 2221.23 & 2106.44 & 30.52 \\
      \hline
    \end{tabular}
    \\
    $300 \times 300$ & $400 \times 400$
  \end{tabular}
\end{minipage}
\begin{minipage}{1.0\linewidth}
  \scriptsize\centering
  \begin{tabular}{cc}
    \begin{tabular}{r||r|r|r|r|>{\bf}r}
      $\epsilon$ & PG & ASPG & ANES & SCOV & SICE\\
      \hline
      1E-01 &25.28 & 153.41 & 209.90 & 218.42 & 9.55 \\
      1E-02 &66.44 & 262.69 & 428.14 & 437.31 & 16.42 \\
      1E-03 &307.73 & 731.48 & 901.46 & 911.02 & 29.82 \\
      1E-04 &631.33 & 1318.33 & 1788.93 & 1828.57 & 45.81 \\
      2E-05 &815.42 & 1615.56 & 3069.17 & 3117.44 & 61.81 \\
      1E-05 &- &1719.70 & 3821.50 & 3853.45 & 66.20 \\
      \hline
    \end{tabular}
    &
    \begin{tabular}{r||r|r|r|r|>{\bf}r}
      $\epsilon$ & PG & ASPG & ANES & SCOV & SICE\\
      \hline
      1E-01 &28.36 & 275.36 & 391.94 & 379.74 & 16.02 \\
      1E-02 &87.15 & 499.69 & 744.94 & 728.40 & 27.69 \\
      1E-03 &244.11 & 1082.84 & 1446.27 & 1365.46 & 49.70 \\
      1E-04 &570.52 & 2088.16 & 2777.44 & 2660.26 & 75.53 \\
      2E-05 &848.27 & 2952.95 & 5049.04 & 4873.74 & 88.16 \\
      1E-05 &972.13 & 3379.78 & 6383.01 & 6130.96 & 99.84 \\
      \hline
    \end{tabular}
    \\
    $500 \times 500$ & $600 \times 600$
  \end{tabular}
\end{minipage}
\begin{minipage}{1.0\linewidth}
  \scriptsize\centering
  \begin{tabular}{cc}
    \begin{tabular}{r||r|r|r|r|>{\bf}r}
      $\epsilon$ & PG & ASPG & ANES & SCOV & SICE\\
      \hline
      1E-01 &96.06 & 291.24 & 422.53 & 327.12 & 28.36 \\
      1E-02 &271.98 & 520.65 & 796.86 & 618.35 & 54.58 \\
      1E-03 &496.38 & 925.59 & 1480.13 & 1155.40 & 80.02 \\
      1E-04 &863.83 & 1332.72 & 2878.07 & 2251.48 & 118.62 \\
      2E-05 &1133.74 & 2042.25 & 4776.05 & 3736.78 & 134.46 \\
      1E-05 &1162.10 & 2077.83 & 6106.81 & 4614.11 & 141.79 \\
      \hline
    \end{tabular}
    &
    \begin{tabular}{r||r|r|r|r|>{\bf}r}
      $\epsilon$ & PG & ASPG & ANES & SCOV & SICE\\
      \hline
      1E-01 &110.74 & 609.90 & 956.90 & 942.62 & 30.17 \\
      1E-02 &320.52 & 880.04 & 1742.28 & 1720.76 & 46.54 \\
      1E-03 &785.29 & 1742.65 & 3106.48 & 3099.92 & 73.06 \\
      1E-04 &1742.95 & 3287.96 & 5931.33 & 5925.93 & 124.42 \\
      2E-05 &2523.65 & 4557.51 & 10083.89 & 10018.28 & 175.97 \\
      1E-05 &- &5470.64 & 12656.09 & 12494.52 & 199.07 \\
      \hline
    \end{tabular}
    \\
    $700 \times 700$ & $800 \times 800$
  \end{tabular}
\end{minipage}
\begin{minipage}{1.0\linewidth}
  \scriptsize\centering
  \begin{tabular}{cc}
    \begin{tabular}{r||r|r|r|r|>{\bf}r}
      $\epsilon$ & PG & ASPG & ANES & SCOV & SICE\\
      \hline
      1E-01 &88.40 & 635.09 & 1306.82 & 1217.88 & 32.21 \\
      1E-02 &245.25 & 1153.18 & 2331.24 & 2158.39 & 63.42 \\
      1E-03 &816.69 & 2862.34 & 4053.39 & 3869.56 & 108.35 \\
      1E-04 &1604.22 & 5424.09 & 7216.28 & 7538.87 & 184.97 \\
      2E-05 &2117.68 & 6974.25 & 11973.47 & 12710.25 & 211.80 \\
      1E-05 &2117.68 & 7907.47 & 15083.79 & 15784.19 & 242.39 \\
      \hline
    \end{tabular}
    &
    \begin{tabular}{r||r|r|r|r|>{\bf}r}
      $\epsilon$ & PG & ASPG & ANES & SCOV & SICE\\
      \hline
      1E-01 &128.61 & 1005.98 & 1818.00 & 1858.70 & 66.08 \\
      1E-02 &358.90 & 1712.01 & 3173.55 & 3216.55 & 102.48 \\
      1E-03 &833.86 & 3529.29 & 5411.49 & 5085.51 & 170.25 \\
      1E-04 &1405.56 & 5008.81 & 9528.91 & 9046.04 & 234.18 \\
      2E-05 &- &6420.18 & 15964.26 & 15428.23 & 307.67 \\
      1E-05 &- &7562.52 & 19812.06 & 19621.27 & 336.25 \\
      \hline
    \end{tabular}
    \\
    $900 \times 900$ & $1000 \times 1000$
  \end{tabular}
\end{minipage}

  
\begin{tabular}{cc}
  \includegraphics[scale=0.4]{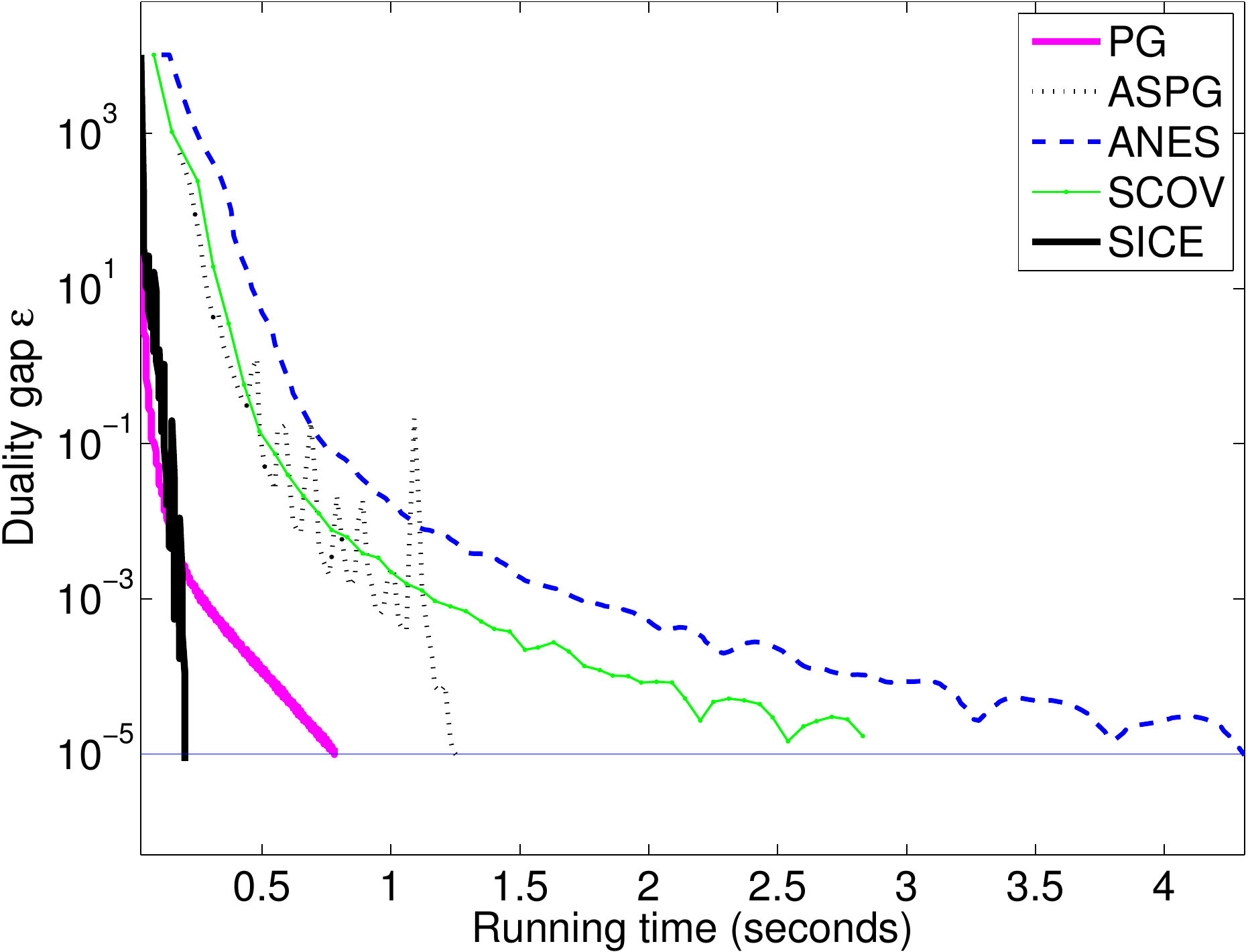} &
  \includegraphics[scale=0.4]{synth_2} \\
  $n=100$ & $n=200$
\end{tabular}

\begin{tabular}{cc}
  \includegraphics[scale=0.4]{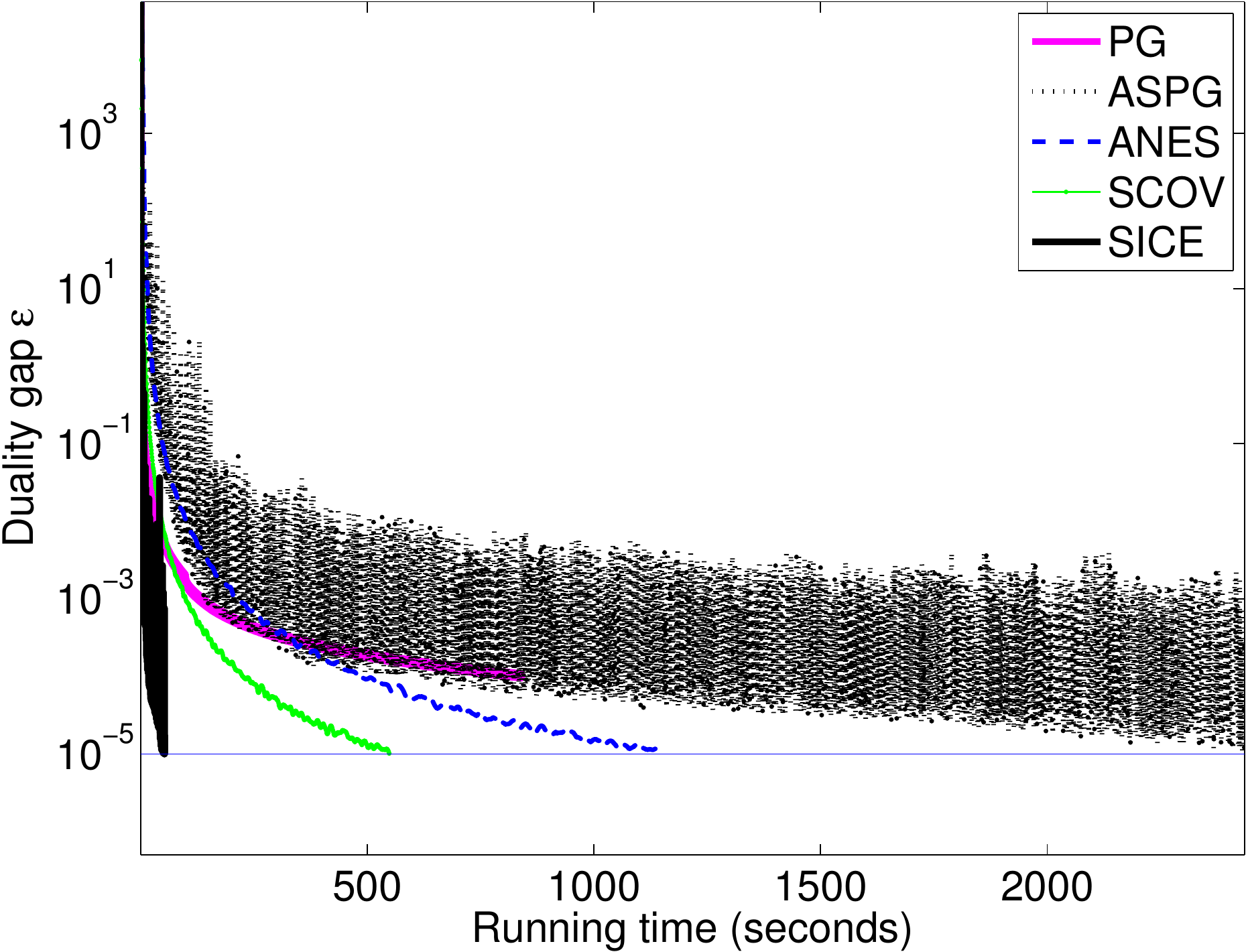} &
  \includegraphics[scale=0.4]{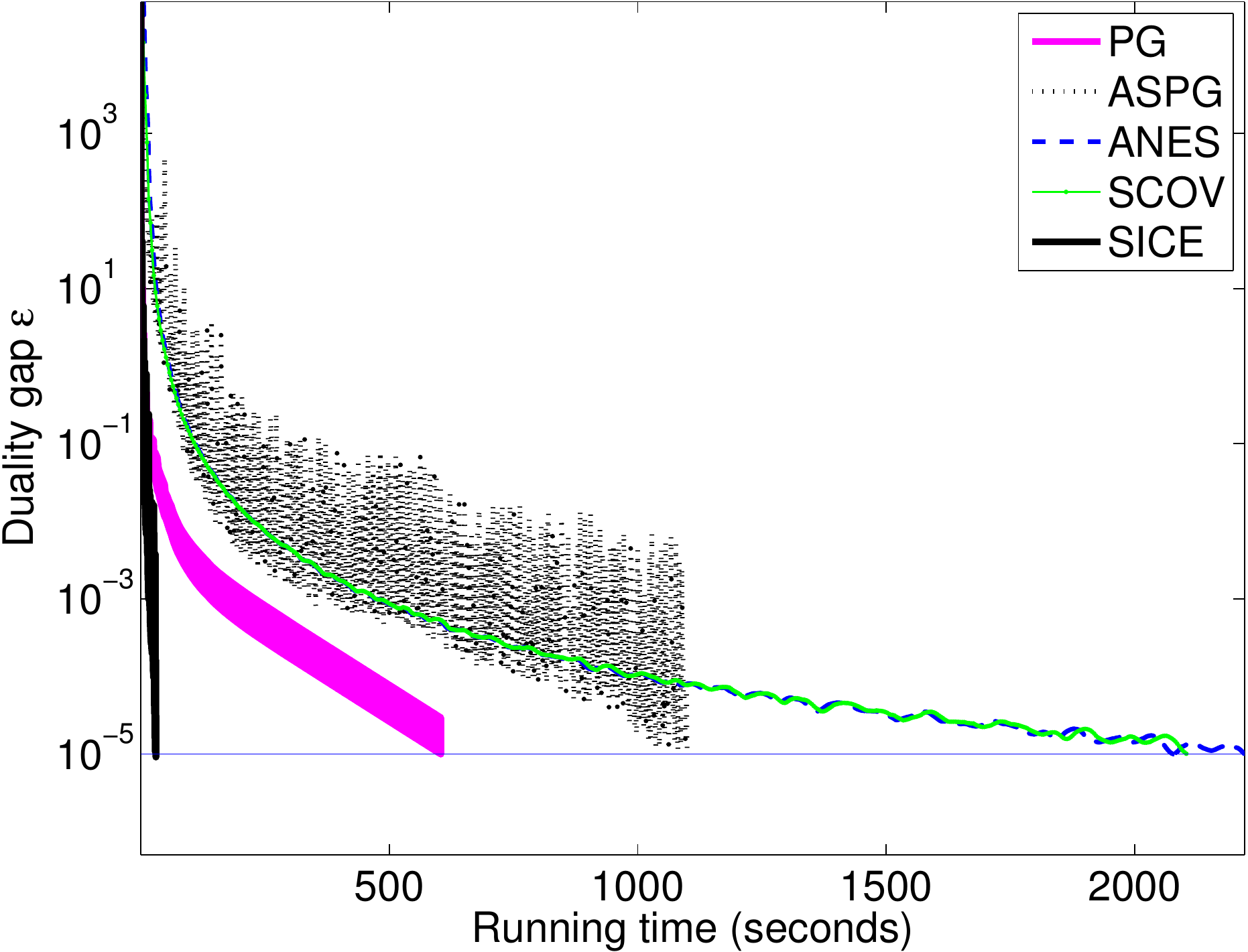}\\
  $n=300$ & $n=400$
\end{tabular}

\begin{tabular}{cc}
  \includegraphics[scale=0.4]{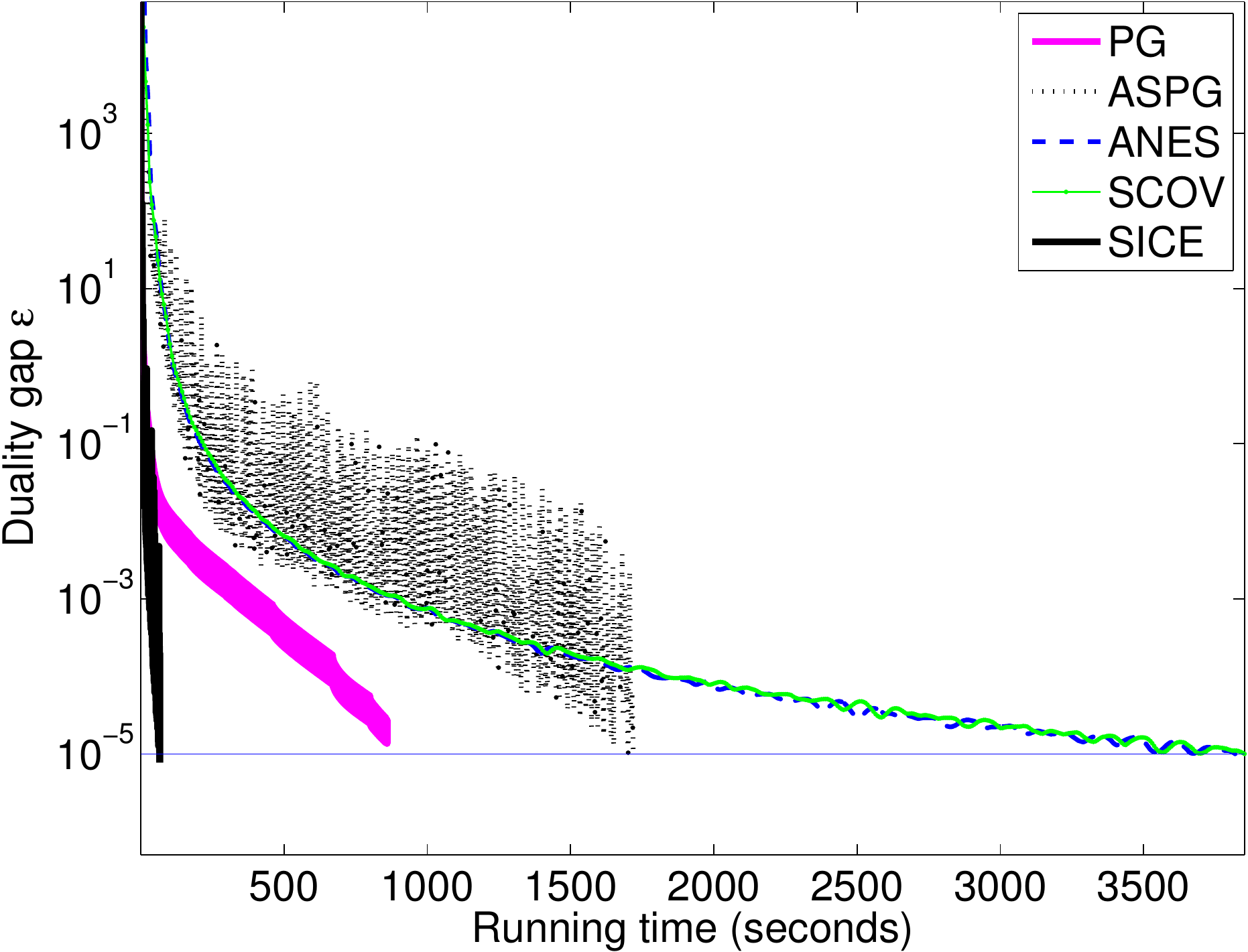} &
  \includegraphics[scale=0.4]{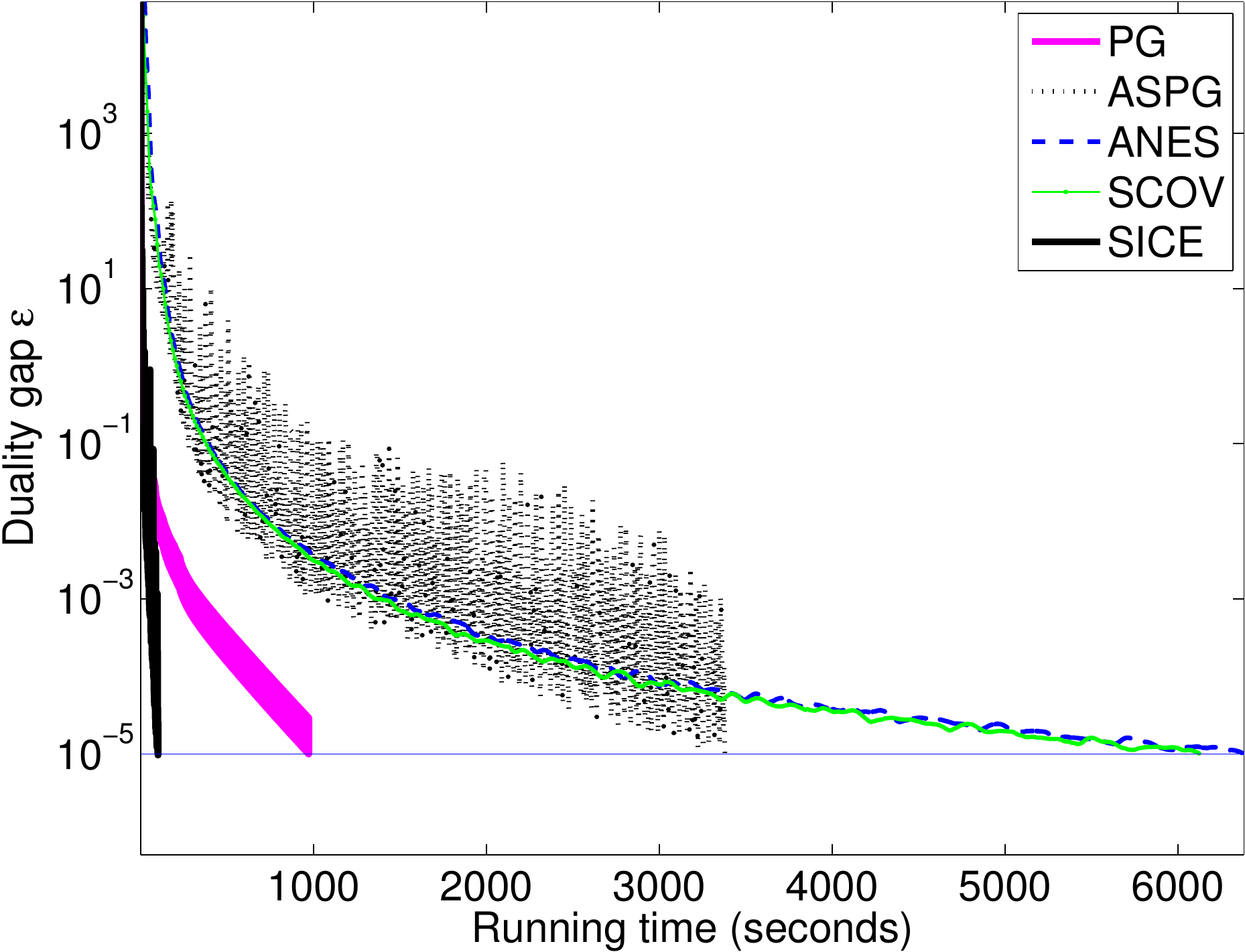}\\
  $n=500$ & $n=600$
\end{tabular}

\begin{tabular}{cc}
  \includegraphics[scale=0.4]{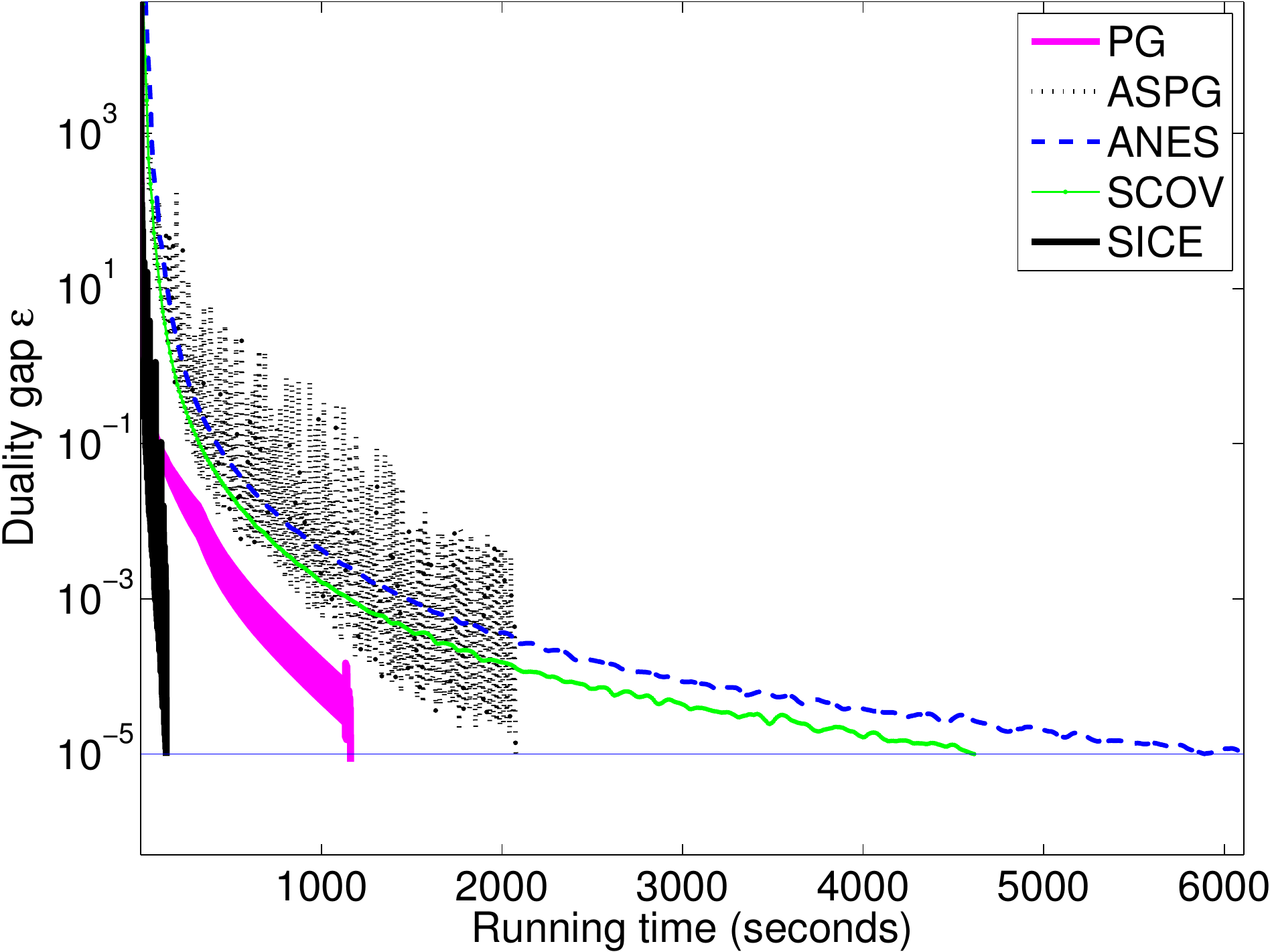} &
  \includegraphics[scale=0.4]{synth_8}\\
  $n=700$ & $n=800$
\end{tabular}

\begin{tabular}{cc}
  \includegraphics[scale=0.4]{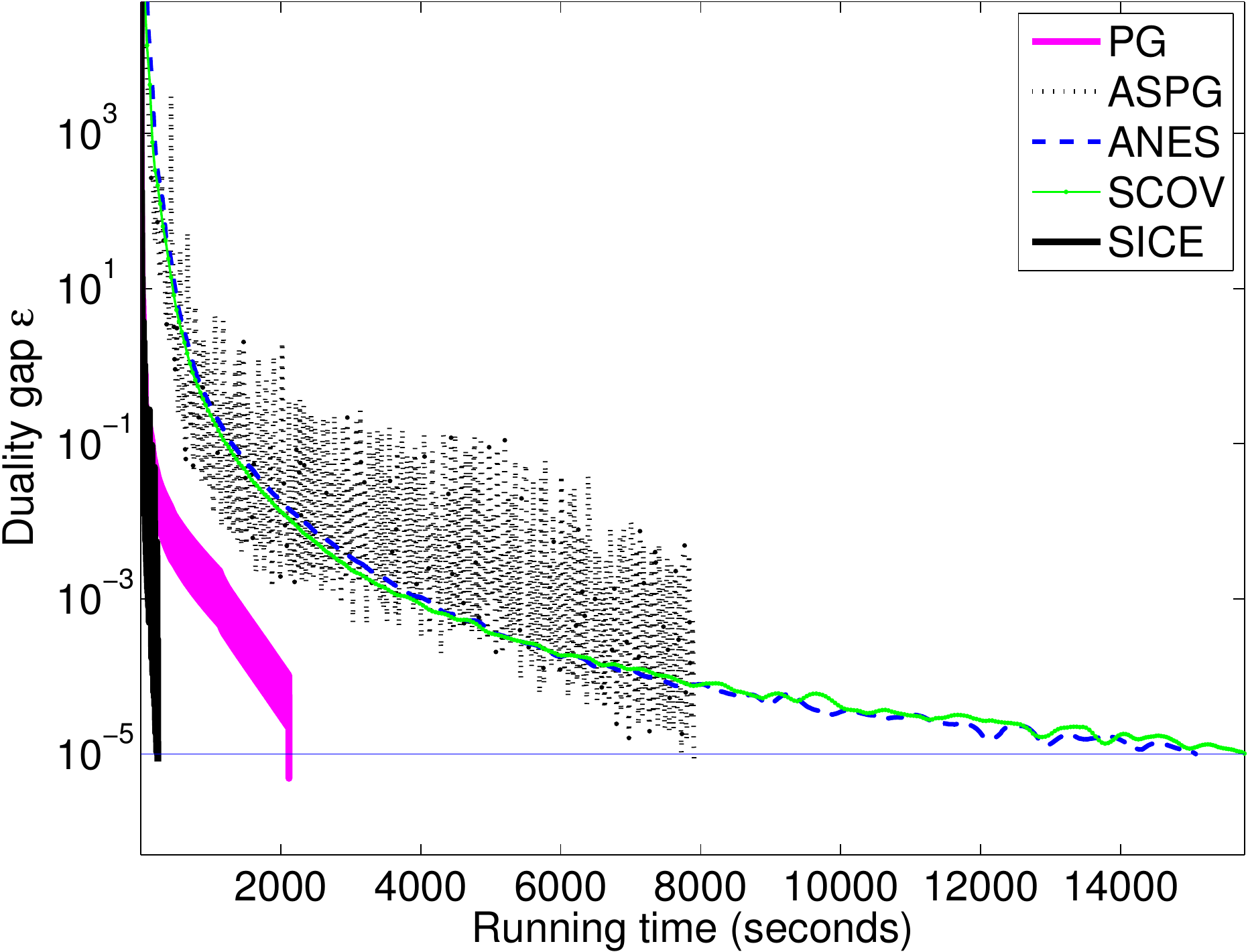} &
  \includegraphics[scale=0.4]{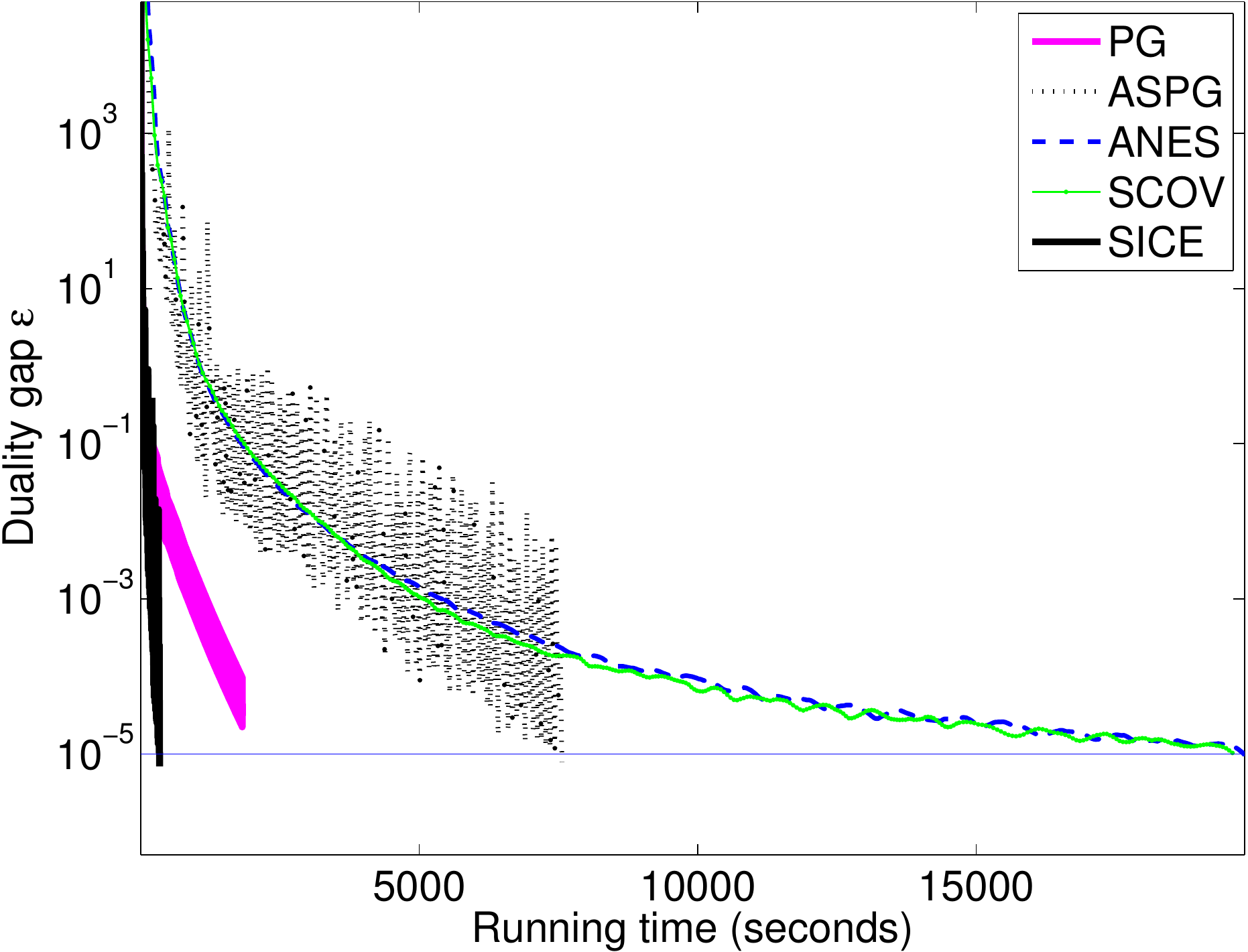}\\
  $n=900$ & $n=1000$
\end{tabular}

\end{document}